\newif\iffull
\fulltrue
\iffull
\documentclass[11pt,twoside]{article}
\usepackage{fullpage,amsmath,amsfonts,amsthm}
\usepackage[style=alphabetic,backend=bibtex,maxbibnames=20,maxcitenames=6,firstinits=true,doi=false,url=false]{biblatex}
\newcommand*{\citet}[1]{\AtNextCite{\AtEachCitekey{\defcounter{maxnames}{2}}} \textcite{#1}}

\newcommand*{\citep}[1]{\cite{#1}}

\bibliography{vf-allrefs-local,stable}
\usepackage{times}
\else

\documentclass[final,12pt]{colt2019} % Anonymized submission
\usepackage{times}
\fi

\usepackage{paralist}

\usepackage{amsmath,bbm,color,hyperref}
\usepackage{vfmacros,framed,algorithm,algorithmic}

\newif\ifnotes
\notestrue
\ifnotes
        \newcommand{\mnote}[1]{{\sf \textcolor{blue}{#1}}}
\else
		\newcommand{\mnote}[1]{}
\fi

\providecommand{\K}{{\mathcal K}}

\providecommand{\F}{{\mathcal F}}
\providecommand{\cE}{{\mathcal E}}
\providecommand{\cP}{{\mathcal P}}

\providecommand{\cU}{{\mathcal U}}

\newcommand{\clamp}{\mathtt{clamp}}
\newcommand{\proj}{\mathtt{proj}}

\newcommand{\bs}{{\bar{s}}}
\newcommand{\bbeta}{{\bar{\eta}}}

\newcommand{\DM}{\Delta_{\bs}(M)}

\newcommand*{\loE}[1]{{\mathcal E}_\bs^{\gets z}[{#1}]}
\newcommand*{\looE}[3]{{\mathcal E}_{#2}^{\gets #3}[{#1}]}
\iffull
\title{High probability generalization bounds for uniformly stable algorithms with nearly optimal rate\footnotetext{Accepted for presentation at Conference on
Learning Theory (COLT) 2019.}}
\author{Vitaly Feldman\thanks{Part of this work was done while the author was visiting the Simons Institute for the Theory of Computing.} \\Google Brain \and Jan Vondr\'ak \\ Stanford University}
\else
\title[High probability generalization bounds for uniformly stable algorithms]{High probability generalization bounds for uniformly stable algorithms with nearly optimal rate}

\coltauthor{%
 \Name{Vitaly Feldman}\thanks{Part of this work was done while the author was visiting the Simons Institute for the Theory of Computing.}\\
 \addr Google Brain
 \AND
 \Name{Jan Vondr\'ak}\\
 \addr Stanford University
}

\fi

\begin{document}

\iffull
\date{}
\fi

\maketitle

\begin{abstract}
Algorithmic stability is a classical approach to understanding and analysis of the generalization error of learning algorithms. A notable weakness of most stability-based generalization bounds is that they hold only in expectation. Generalization with high probability has been established in a landmark paper of Bousquet and Elisseeff (2002) albeit at the expense of an additional $\sqrt{n}$ factor in the bound. Specifically, their bound on the estimation error of any $\gamma$-uniformly stable learning algorithm on $n$ samples and range in $[0,1]$ is $O(\gamma \sqrt{n \log(1/\delta)} + \sqrt{\log(1/\delta)/n})$ with probability  $\geq 1-\delta$. The $\sqrt{n}$ overhead makes the bound vacuous in the common settings where $\gamma \geq 1/\sqrt{n}$. A stronger bound was recently proved by the authors (Feldman and Vondrak, 2018) that reduces the overhead to at most $O(n^{1/4})$. Still, both of these results give optimal generalization bounds only when $\gamma = O(1/n)$.

We prove a nearly tight bound of $O(\gamma \log(n)\log(n/\delta) + \sqrt{\log(1/\delta)/n})$ on the estimation error of any $\gamma$-uniformly stable algorithm. It implies that for algorithms that are uniformly stable with $\gamma = O(1/\sqrt{n})$, estimation error is essentially the same as the sampling error. Our result leads to the first high-probability generalization bounds for multi-pass stochastic gradient descent and regularized ERM for stochastic convex problems with nearly optimal rate --- resolving open problems in prior work. Our proof technique is new and we introduce several analysis tools that might find additional applications.
\end{abstract}

\section{Introduction}
We consider the following problem. Let $\bs=(s_1,\ldots,s_n) \in Z^n$ be a dataset over an arbitrary domain and $M \colon Z^n \to [0,1]^Z$ be an arbitrary algorithm (or mapping) from datasets to functions over $Z$ with range in $[0,1]$. $M$ is said to be $\gamma$-uniformly stable if for all
datasets $\bs$ and $\bs'$ that differ in a single element $\|M(\bs)-M(\bs')\|_\infty \leq \gamma$. Equivalently, for every $z \in Z$, $|M(\bs,z)-M(\bs',z)|\leq \gamma$ (where $M(\bs,z)$ refers to the value of the function $M(\bs)$ on $z$). Assume that $\bs$ consists of samples drawn i.i.d. from some distribution $\cP$ over $Z$. We address the question of how well the true expectation of $M(\bs)$ on $\cP$, that is $\E_\cP[M(\bs)] = \E_{z\sim \cP}[M(\bs,z)]$ is approximated by the empirical mean of $M(\bs)$ on $\bs$, that is $\cE_\bs[M(\bs)] = \fr{n} \sum_{i\in [n]} M(\bs,s_i)$. The value $$\DM \doteq |\E_\cP[M(\bs)] -\cE_\bs[M(\bs)]|$$ is referred to as the {\em estimation error} of $M$ at $\bs$.

The primary motivation and the origin of this question is understanding of the generalization error of learning algorithms that are uniformly stable. In this context, $Z =X\times Y$ is labeled points and the goal is to analyze a learning algorithm  $A$ that given $\bs$ outputs a model $f_\bs \colon X \to Y$. The output of the learning algorithm is evaluated via some loss function $\ell_Y\colon Y\times Y \to \R_+$, with true loss being defined as $\E_{(x,y)\sim \cP}[\ell_Y(f_\bs(x),y)]$. By defining $M(\bs,(x,y)) = \ell_Y(f_\bs(x),y)$ we get that the estimation error of $M$ is exactly the difference between the true loss of $f_\bs$ and the empirical loss of $f_\bs$ on $\bs$ (sometimes referred to as the {\em generalization gap}).

Stability is a classical approach to proving generalization bounds pioneered by %Devroye, Rogers and Wagner
\citet{RogersWagner78,DevroyeW79,DevroyeW79a}. It is based on analysis of the sensitivity of the learning algorithm to changes in the dataset such as leaving one of the data points out or replacing it with a different one. The choice of how to measure the effect of the change and various ways to average over multiple changes give rise to a variety of stability notions that have been examined in the literature (e.g. \citep{BousquettE02,MukherjeeNPR06,ShwartzSSS10}). Unfortunately, most stability notions only lead to bounds on the expectation or the second moment of the estimation error over the random choice of the dataset. In contrast, generalization bounds based on uniform convergence show that the estimation error is small with high probability (more formally, the distribution of the error has exponentially decaying tails). Beyond theoretical interest, high-probability generalization bounds are necessary for inferring strong generalization bounds whenever the algorithm is used many times, either on its own or as a subroutine in another algorithm. For example, in practice several hyperparameters are usually tuned with the best result chosen based on the value of the empirical error. Similarly, for optimization algorithms a stopping condition based on the empirical error is often used.

High probability bounds for the estimation error based on stability were studied already in \citep{DevroyeW79} who obtained them for the $k$-Nearest Neighbor algorithm. Tighter bounds for the algorithm and several additional ones were obtained by \citet{LugosiPawlak94}. In a seminal work \citet{BousquettE02} developed a general approach based on the notion of {\em uniform stability} (defined above). While uniform stability is a relatively strong condition, it is satisfied by several well-studied algorithms. For example, for strongly convex Lipschitz losses the ERM is uniformly stable \citep{BousquettE02,ShwartzSSS10} (we describe the bounds quantitatively in Sec.~\ref{sec:apps}). More recently, \citet{HardtRS16} showed that for convex smooth losses the solution obtained via gradient descent is uniformly stable, allowing them to give the first generalization guarantees for many variants of (stochastic) gradient descent (SGD). Importantly, no other known approaches give comparable generalization bounds for these fundamental algorithms. In particular, for several standard stochastic convex optimization problems the best bound achievable via uniform convergence is worse than the bound obtained via stability of SGD by a $\Omega(\sqrt{d})$ factor, where $d$ is the dimension of the problem \citep{ShwartzSSS10,Feldman:16erm}. This implies that approaches requiring uniform convergence over the set of all models that minimize the empirical loss (such as most model-complexity-based bounds) will not lead to useful generalization guarantees in this case. We remark that continuous optimization methods play a central role in modern machine learning and hence their generalization properties is a topic of intense theoretical and practical interest in recent years.

\subsection{Prior work}
The main generalization bound for $\gamma$-uniformly stable algorithms given in \citep{BousquettE02} states that for some constant $c_0$,
 \equ{ \pr_{\bs\sim \cP^n}\lb \DM \geq c_0 \lp  \gamma \sqrt{n}  + \frac{1}{\sqrt{n}} \rp \sqrt{\log (1/\delta)} \rb \leq  \delta \label{eq:hp}.}
This is in contrast to an easy observation that the expectations of $\E_\cP[M(\bs)]$ and  $\cE_\bs[M(\bs)]$ are within $\gamma$. Namely,
 \equ{\left| \E_{\bs\sim \cP^n}\lb \E_\cP[M(\bs)] -\cE_\bs[M(\bs)] \rb \right| \leq \gamma \label{eq:exp}.}
Thus the bound on estimation error is worse by at least a factor of $\sqrt{n}$ than the expected difference. In terms of lower bounds, note that the term  $\frac{\sqrt{\log (1/\delta)}}{\sqrt{n}}$ is necessary since even for an algorithm that outputs a fixed function (or $\gamma=0$) this is the optimal bound on the sampling error. In addition, estimation error is at least $\gamma$ since the function can change arbitrarily in this range.

Naturally, for most algorithms the stability parameter needs to be balanced against the guarantees on the empirical loss. For example, ERM solution to convex learning problems can be made uniformly stable by adding a strongly convex term to the objective \citep{ShwartzSSS10}. This change in the objective introduces an error that may increase the original empirical loss. In the other example, the stability parameter of gradient descent on smooth objectives is determined by the sum of the rates used for all the gradient steps \citep{HardtRS16}. Limiting the sum limits the empirical loss that can be achieved. In both of those examples the optimal expected loss is achieved when $\gamma = \Theta(1/\sqrt{n})$. Unfortunately, in this setting, eq.~\eqref{eq:hp} gives a vacuous bound. As a result, in these applications only bounds on the expectation of the true loss are stated. For both of these applications, deriving a high-probability generalization bound is stated as an open problem \citep{ShwartzSSS10,HardtRS16}.

Note that eq.~\eqref{eq:exp} does not imply that $\E_\cP[M(\bs)] \leq \cE_\bs[M(\bs)] + O(\gamma/\delta)$ with probability at least $1-\delta$ since $\E_\cP[M(\bs)] -\cE_\bs[M(\bs)]$ can be negative and Markov's inequality cannot be used. Such ``low-probability'' generalization was first derived by \citet{ShwartzSSS10} for learning algorithms that minimize the empirical risk. For such algorithms they showed that
\equ{\E_{\bs\sim \cP^n}\lb \DM \rb  \leq O\lp\gamma + \fr{\sqrt{n}}\rp , \label{eq:first-moment}} allowing them to apply Markov's inequality. %They raises a natural question of whether the known bounds in eq.~\eqref{eq:var} and eq.~\eqref{eq:hp} are optimal.

Generalization properties of uniform stability were addressed in a recent work by the authors \citep{FeldmanV:18}. There we demonstrated that there exists a constant $c_1$ such that \equ{ \pr_{\bs\sim \cP^n}\lb \DM \geq c_1 \lp  \sqrt{\gamma} + \frac{1}{\sqrt{n}} \rp \sqrt{\log (1/\delta)} \rb \leq  \delta \label{eq:hp_fv}} improving on eq.~\eqref{eq:hp} for $\gamma = \omega(1/n)$. This result reduces the overhead of high-probability generalization from $\sqrt{n}$ to at most $n^{1/4}$ (achieved for $\gamma = 1/\sqrt{n}$). This bound was used to strengthen the generalization guarantees that are known for the convex optimization algorithms described above but only implies that suboptimality of the solution is $O(1/n^{1/3})$ with high-probability (whereas the optimal rate is $O(1/\sqrt{n})$).

Further, we gave an optimal (up to constant factors) bound on the second moment of the estimation error:
\equn{\E_{\bs\sim \cP^n}\lb \DM^2 \rb  \leq O\lp\gamma^2 + \fr{n}\rp,}
improving on the $O(\gamma + \fr{n})$ bound in \citep{BousquettE02}.

A natural question of whether the high-probability bounds can be strengthened (or a matching lower bound can be proved) still remained open.

\subsection{Our contribution}
Our main result is a high-probability generalization bound for any $\gamma$-uniformly stable algorithm that has only a logarithmic overhead. In particular, it gives an exponential improvement (in terms of the tail bound $\delta$) over prior work.
\begin{thm}
\label{thm:main-intro}
Let $M:Z^n \times Z \to [0,1]$ be an algorithm  (or a data-dependent function) with uniform stability $\gamma$. Then there exists a constant $c$ such that for any probability distribution $\cP$ over $Z$ and any $\delta \in (0,1)$:
 \equn{ \pr_{\bs\sim \cP^n}\lb \DM \geq c\lp \gamma \log(n) \log (n/\delta)   + \frac{\sqrt{\log (1/\delta)}}{\sqrt{n}} \rp
 %32 \gamma \ln (5n^3/\delta)  \log_2 (n) + \frac{2}{\sqrt{n}}  \sqrt{\ln (4/\delta)}
 \rb
 %\rp  \rb
 \leq \delta \label{eq:hp-new}.}
\end{thm}
A somewhat surprising implication of this result is that algorithms that are uniformly stable with $\gamma = O(1/\sqrt{n})$ enjoy essentially the same estimation error guarantees as algorithms that do not look at the data and output a fixed function. For
$\gamma \leq \sqrt{\log(1/\delta)} / (\sqrt{n} \log (n/\delta) \log (n))$, there is no significant contribution depending on $\gamma$ and our bound is optimal up to constant factors. In contrast, both previous works \citep{BousquettE02,FeldmanV:18} give similar generalization guarantees only when $\gamma = O(1/n)$.

\paragraph{Proof approach:}
The high-probability generalization result in \citep{BousquettE02} (eq.~\eqref{eq:hp}) is based on a simple observation that as a function of $\bs$, the estimation error has sensitivity of at most $2\gamma + 1/n$. Applying McDiarmid's concentration inequality immediately implies concentration with standard deviation of $\sqrt{n}(\gamma + 1/n)$ around the expectation. The expectation, in turn, is at most $\gamma$ by eq.~\eqref{eq:exp}.

The approach in our prior work \citep{FeldmanV:18} is based on a technique developed in \citep{BassilyNSSSU16} to prove generalization bounds for differentially private algorithms. It bounds the tail by proving a bound on the expectation of the maximum of many independent copies of the estimation error. The latter is bounded by using a soft-argmax operation. Soft-argmax is itself stable and hence the expectation of the estimation error of the copy it outputs is small. While the bound of $\sqrt{\gamma}$ derived using this approach may appear to be arbitrary, it has been re-derived using other approaches by the authors and also by \citet{WeinbergerRakhlin18} who used a bound on the second moment from \citep{FeldmanV:18} to bound the moment generating function of the estimation error.

Our approach is based on two new ideas that both rely strongly on the structure of the estimation error. The first idea is to upper bound the estimation error by using the bound on the estimation error over a smaller dataset. This step is very simple technically and can already be used to re-derive the $\sqrt{\gamma}$ bound from our earlier work \citep{FeldmanV:18} (optimizing the simple bound $\gamma \sqrt{n'} + 1/\sqrt{n'}$ over $n' \leq n$ gives exactly $2\sqrt{\gamma}$).

The second idea is to reduce the range or the output function by subtracting the mean and ``clamping" the values outside the range. Uniform stability can be used to ensure that for an appropriately chosen range this procedure will introduce only a small error. The main technical issue is that we need to ensure that the clamping procedure both preserves the stability parameter and does not shift the mean of the estimation error (as the first step requires a zero-mean random variable). Achieving both of these goals requires a more involved ``clamping" procedure and delicate analysis.

Combining these procedures decomposes the estimation error into a sum of mixtures of ``local" approximations (that is, accurate for specific setting of some of the samples in the dataset). Repeated application of this combination in a recursive way gives the proof of our main result. The $\log n$ levels of recursion are the reason for the $\log n$ overhead of our bound. In Sec.~\ref{sec:overview} we give a more technical overview of the proof.

\subsection{Applications}
We now apply our bounds on the estimation error to several known uniformly stable algorithms. Our main focus are learning problems that can be formulated as stochastic convex optimization. Specifically, these are problems in which the goal is to minimize the expected loss: $F_\cP(w) \doteq \E_{z\sim \cP}[\ell(w,z)]$ over $w \in \K$ for some convex body $\K \subset \R^d$ and a family of convex losses $\F =\{\ell(\cdot,z)\}_{z\in Z}$. The stochastic convex optimization problem for a family of losses $\F$ over $\K$ is the problem of minimizing $F_\cP(w)$ for an arbitrary distribution $\cP$ over $Z$. For concreteness, we consider the well-studied setting in which $\F$ contains $1$-Lipschitz convex functions with range in $[0,1]$ and $\K$ is included in the unit ball (settings with an arbitrary Lipschitz constant and domain radius can be reduced to this case via scaling).

\paragraph{Strongly convex ERM:}
In this setting with an additional assumption that loss functions in $\F$ are $\lambda$-strongly convex, ERM has uniform stability of $4/(\lambda n)$ \citep{BousquettE02}. We therefore obtain high-probability generalization bounds on ERM in this case that improve on the known results for any $\lambda = o(1)$ (see Corollary \ref{cor:strongly-convex} for details).

Using stability of ERM for strongly convex functions, \citet{ShwartzSSS10} showed that even without strong convexity, the stochastic convex optimization problem can be solved by adding a strongly convex regularizer $\frac{\lambda}{2} \|w\|^2$ to the empirical loss with $\lambda = 1/\sqrt{n}$. They demonstrate that the expected loss of this algorithm is optimal and conjecture that high-probability generalization bounds hold as well. Using Thm.~\ref{thm:main-intro}, we show that the excess loss (or sub-optimality) of the solution is at most $O(\log(n/\delta)/\sqrt{n})$ with probability at least $1-\delta$, thereby proving the conjecture. (The optimal choice of $\lambda= \log (n)/\sqrt{n}$ is determined by balancing the estimation error and the error introduced by adding the regularizer.).
\begin{cor}
\label{cor:convex-general}
Let $\K$ be a convex body of radius $1$, $\F = \{\ell(\cdot, z) \cond z\in Z\}$ be a family of convex $1$-Lipschitz loss functions over $\K$ with range in $[0,1]$. For a dataset $\bs \in Z^n$ let $w_\bs$ denote the empirical minimizer of regularized loss on $\bs$: $w_{\bs,\lambda} =\argmin_{w \in \K} \sum_{i\in [n]}\ell(w,s_i) + \frac{\lambda n}{2} \|w\|_2^2$. There exist a constant $c$ such that for every distribution $\cP$ over $Z$, $\delta >0$ and $\lambda=\log(n)/\sqrt{n}$:
$$\pr_{\bs\sim \cP^n} \lb F_\cP(w_{\bs,\lambda}) \geq \min_{w \in \K} F_\cP(w) + \frac{c \log (n/\delta)}{\sqrt{n}}\rb \leq \delta .$$
\end{cor}

\paragraph{(Stochastic) gradient descent:}
Another fundamental application of uniform stability is proving generalization bounds for  (stochastic) gradient descent on sufficiently smooth convex loss functions \citep{HardtRS16}. Importantly, in this case the estimation error can be bounded without any assumptions on how close the output of the algorithm is to the empirical minimum. Therefore this approach can be used to give generalization bounds for variants of SGD used in practice (as opposed to those prescribed by theoretical analysis). For most versions of SGD no alternative analyses of the estimation error are known. The analysis in \citep{HardtRS16} focuses on the stochastic gradient descent and derives uniform stability for the expectation of the loss (over the randomness of the algorithm). From this result they obtain generalization in expectation over both randomness of the algorithm and the choice of the dataset. Obtaining bounds that hold with high-probability was left as an open problem.

Theorem~\ref{thm:main-intro} ensures that the bounds on estimation error hold with high probability over the choice of the dataset. This suffices to get generalization with high probability for deterministic variants of gradient descent. As an example application, we derive nearly optimal generalization bounds for full gradient descent (see Corollary \ref{cor:smooth}). To obtain generalization bounds for SGD we additionally observe that for most standard choices of picking batches randomly, the uniform stability of the gradient descent as a function of the randomness of SGD is highly concentrated around its mean. As a result we can obtain a bound on the estimation error that holds with high probability over the randomness of SGD and is worse than the bound that holds in expectation by at most a logarithmic factor. As an example application of this technique we derive nearly optimal generalization bounds for stochastic gradient descent that uses sampling with replacement for each gradient and batch size of 1 (see Corollary~\ref{cor:psgd-resample}).

For comparison, a recent work of \citet{London17} considers extension of the generalization guarantees in \citep{HardtRS16} to high-probability over the randomness in the choice of samples. The approach there relies on sensitivity of the estimation error to the choices of random samples. It requires independent sampling at each step and the resulting bound on the estimation error has an overhead of $\sqrt{T}$, where $T$ is the number of iterations. As a result it gives much weaker bounds in the setting we consider (\citep{London17} focuses on the smooth and strongly convex case).

\paragraph{Prediction privacy:}
Finally, we show that our results can be used to improve the recent bounds on estimation error of learning algorithms with differentially private prediction. These are algorithms introduced to model privacy-preserving learning in the settings where users only have black-box access to the learned model via a prediction interface \citep{DworkFeldman18} (see Def.~\ref{def:private-prediction}).
%(see Def.~\ref{def:private-prediction}).
The properties of differential privacy imply that the expectation over the randomness of a predictor $K\colon (X\times Y)^n \times X$ of the loss of $K$ at any point $x \in X$ is uniformly stable. Specifically, for an $\eps$-differentially private prediction algorithm, every loss function $\ell_Y\colon Y\times Y \to [0,1]$, two datasets $\bs,\bs'\in (X\times Y)^n$ that differ in a single element and $(x,y) \in X\times Y$:
$$\left| \E_K[\ell_Y(K(\bs,x),y)] - \E_M[\ell_Y(K(\bs',x),y)] \right| \leq e^\eps-1 .$$
Therefore, our generalization bounds can be directly applied to the data-dependent function $M(\bs, (x,y)) \doteq \E_K[\ell_Y(K(\bs,x),y)]$. These bounds
 can, in turn, be used to get nearly optimal generalization bounds for an algorithm for learning linear thresholds given in \citep{DworkFeldman18} (that relies on models of unbounded complexity). The details of these applications appear in Section \ref{sec:apps}.
%While the bound $\sqrt{\gamma}$ may appear arbitrary it has been derived via several alternative approaches

\subsection{Other related work}
Early work on stability focused on obtaining generalization guarantees for ``local'' algorithms such as $k$-nearest neighbor. The bounds were also primarily on variance of the estimation error (a notable exception is \citep{DevroyeW79} where high probability bounds on the generalization error of $k$-NN are proved). See \citep{DevroyeGL:96book} for an overview. Stability is also used in a similar spirit for bounding the estimation error of other estimators of true loss such as leave-one-out and $k$-fold cross-validation estimators (for example \citep{BlumKL99,KaleKV11,KumarLVV13}).

A long line of work focuses on the relationship between various notions of stability and learnability in supervised setting (see  \citep{KearnsRon99,PoggioRMN04,ShwartzSSS10} for an overview). This work employs relatively weak notions of average stability and derives a variety of asymptotic equivalence results. The results in \citep{BousquettE02} on uniform stability and their applications to generalization properties of strongly convex ERM algorithms have been extended and generalized in several directions (e.g.~\citep{Zhang03,WibisonoRP09}). \citet{Maurer17} considers generalization bounds for a special case of linear regression with a strongly convex regularizer and a sufficiently smooth loss function. Their bounds are data-dependent and are potentially stronger for large values of the regularization parameter (and hence stability). However the bound is vacuous when the stability parameter is larger than $n^{-1/4}$ and hence is not directly comparable to ours.
\citet{KuzborskijL18} give data-dependent generalization bounds for SGD on smooth convex and non-convex losses based on stability. They use on-average stability that does not imply generalization bounds with high probability.

Recent work of \citet{Abou-MoustafaS18} and \citet{celisse2016stability} gives high probability generalization bounds similar to those in \citep{BousquettE02} but using a bound on a high-order moment of stability instead of the uniform stability. Several works demonstrate that stability-based analysis of generalization can be combined with uniform convergence and PAC-Bayes bounds leading to a more general overall technique \citep{LiuLNT17,rivasplata2018pac,FosterGKLMS19}. The analysis of \citet{FosterGKLMS19} builds on the technique in \citep{FeldmanV:18} and demonstrates that this combined approach to generalization can benefit from better analyses of uniform stability. Recent applications of stability to generalization can be found for example in \citep{KorenLevy15,CharlesP18,chen2018stability}. %We also remark that all these works are based on techniques different from ours.

Uniform stability has several additional important connections to differential privacy \citep{DworkMNS:06}. First, differential privacy is itself a type of worst-case stability guarantee that bounds the effect of every data point on the output distribution of the algorithm. Our work is in part inspired by the recent progress showing that differential privacy implies generalization with high probability \citep{DworkFHPRR14:arxiv,BassilyNSSSU16}. Both the assumptions and guarantees given in this line of work are different from ours and we do not know a way to relate between those. For example, the generalization guarantees obtained in work on differential privacy hold with high probability over the randomness of the algorithm, whereas our results when applied to a differentially private algorithm would only give generalization of the expectation over the algorithm's randomness. We remark that the techniques developed in this line of work were used to re-derive and extend several standard concentration inequalities \citep{SteinkeU17subg,NissimS17} and also in \citep{FeldmanV:18} to give an improved generalization bound for uniform stability. %Our work relies on unrelated techniques.

Uniformly stable algorithms also play an important role in privacy-preserving learning since a differentially private learning algorithm can usually be obtained by adding noise to the output of a uniformly stable one (e.g.~\citep{ChaudhuriMS11,wu2017bolt,DworkFeldman18}). Hence understanding the generalization properties of uniformly stable algorithms is likely to play an important role in this line of research. 
\section{Preliminaries}
\label{sec:prelims}
For a domain $Z$, a dataset $\bs\in Z^n$ in an $n$-tuple of elements in $Z$. We refer to element with index $i$ by $s_i$ and by $\bs^{i\gets z}$ to the dataset obtained from $\bs$ by setting the element with index $i$ to $z$.  We refer to a function that takes as an input a dataset $\bs \in Z^n$ and a point $z \in Z$ as a {\em data-dependent function} over $Z$.

We think of data-dependent functions as outputs of an algorithm that takes $\bs$ as an input. For example in supervised learning $Z$ is the set of all possible labeled examples $Z= X\times Y$ and the algorithm $M$ is defined as the loss of the model $f_\bs$ output by a learning algorithm $A(\bs)$ on example $z=(x,y)$. That is $M(\bs,z) = \ell_Y(f_\bs(x),y)$ for some loss function $\ell_Y: Y\times Y \to \R_+$. Note that in this setting
$\cE_\cP [M(\bs)]$ is exactly the true loss of $f_\bs$ on data distribution $\cP$, whereas $\cE_\bs[M(\bs)]$ is the empirical loss of $f_\bs$.
\iffull
\begin{defn}
\label{def:stability}
A data-dependent function $M\colon Z^n\times Z \to \R$ has uniform stability $\gamma$ if for all $\bs\in Z^n$, $i\in [n]$, $s,z\in Z$, $|M(\bs,z) - M(\bs^{i\gets s},z)| \leq \gamma$.
\end{defn}
Two natural ways to use this definition for randomized algorithms are $(1)$ consider stability of the expectation over the algorithm's randomness $\E_M[M(S,z)]$; $(2)$ consider stability of $M(S,z)$ for a fixed setting of $M$'s random bits. The first approach is simpler and usually results in a better stability parameter but only leads to generalization guarantees for $\E_M[M(S)]$ (for examples see \cite{ElisseeffEP05,HardtRS16} and Sec.~\ref{sec:dp-api}). This approach is necessary for obtaining a non-trivial bound on uniform stability in classification problems \cite{LugosiPawlak94,BousquettE02}. In contrast, the second approach may lead to generalization with high-probability over $M$'s randomness if the stability parameter can be upper-bounded with high probability (over $M$'s randomness).
\fi

Uniform stability is equivalent to $M(\bs,z)$ as a function of $\bs$ having {\em sensitivity} $\gamma$ (or $\gamma$-bounded differences) for each fixed $z \in Z$.
\begin{defn}
A real-valued function $f:Z^n \to \R$ has sensitivity at most $\gamma$ if for all $\bs\in Z^n$, $i\in [n]$, $s \in Z$, $|f(\bs)-f(\bs^{i\gets s})| \leq \gamma$.
\end{defn}
We will use McDiarmid's inequality for functions of bounded sensitivity.
\begin{lem}
\label{lem:mcdiarmid}
Let $f:Z^n \to \R$ be a function with sensitivity of at most $\gamma$. Then for any distribution $\cP$ over $Z$,  $\mu \doteq \E_{\bs \sim \cP^n}[f(\bs)]$ and any $t > 0$,
$$ \pr_{\bs \sim \cP^n}[f(\bs) \geq \mu  + t] \leq e^{-2t^2 / (n \gamma^2)}.$$
\end{lem}

\paragraph{Estimation error:}
For convenience, as in \citep{FeldmanV:18}, we reduce bounds on $\DM$ to bounds on the leave-one-out estimation error for the unbiased version of $M$ (we include the details here for completeness).

Specifically, we define $L(\bs,z) \doteq M(\bs,z) - \E_{z \sim \cP}[M(\bs,z)]$. Clearly, $L$ is {\em unbiased} with respect to $\cP$ in the sense that for every $\bs\in Z^n$, $\cE_\cP[L(\bs)] = \E_{z \sim P}[L(\bs,z)] = 0$. Note that if the range of $M$ is $[0,1]$ then the range of $L$ is $[-1,1]$. Further, $L$ has uniform stability of at most $2\gamma$ since for two datasets $\bs$ and $\bs'$ that differ in a single element,
$$ |L(\bs,z) - L(\bs',z)| \leq |M(\bs,z) - M(\bs',z)| + \E_{z \sim \cP}[|M(\bs,z) - M(\bs',z)|] \leq 2 \gamma.$$
Observe that \equn{\DM = \left| \fr{n} \sum_{i=1}^n \lp \cE_\cP[M(\bs)] - M(\bs,s_i) \rp  \right| = \left| -\frac{1}{n} \sum_{i=1}^n L(\bs,s_i) \right| =
\left|\cE_\bs[L(\bs)] \right| .\label{eq:unbias}}
The leave-one-out version of the estimation error is defined as follows.
For any $z\in Z$ let
$$\loE{L} := \fr{n} \sum_{i\in [n]} L(\bs^{i \gets z},s_i) .$$
Observe that the uniform stability of $L$ implies that for every $\bs$ and every $z$,
\alequ{\left| \cE_\bs[L(\bs)] - \loE{L} \right| &= \left| \fr{n}\sum_{i\in [n]} L(\bs,s_i)  - \fr{n} \sum_{i\in [n]} L(\bs^{i \gets z},s_i)  \right| \nonumber\\
&\leq \fr{n}\sum_{i\in [n]}  \left| L(\bs,s_i)  - L(\bs^{i \gets z},s_i)  \right| \leq \gamma \label{eq:pointwise}.}

\paragraph{Tail bounding function:}
The goal of our analysis is to bound the following function.

\begin{defn}
For an integer $n$, real $R,\gamma>0$ and $\delta>0$, let $D_\delta(n,R,\gamma)$ be the maximum value $D$ such that for every domain $Z$, probability distribution $\cP$ over $Z$, and a data-dependent $\gamma$-uniformly stable function $L:Z^n \times Z \to [-R,R]$ such that $\E_{z \sim \cP}[L(\bs,z)] = 0$,
$$ \Pr[|\loE{L}| \geq  D] \leq \delta.$$
\end{defn}

Observe that by simple scaling, the range and stability in the definition of $D_\delta$ can be adjusted by an arbitrary factor. That is, for an arbitrary factor $\alpha > 0$,
\equ{D_\delta(n,R,\gamma) = \frac{1}{\alpha} D_\delta(n,\alpha R,\alpha \gamma) .\label{eq:scale}} 
\section{Proof of the Main Result}
In this section, we prove our main concentration bound with exponential tails.% and standard deviation on the order of $(\gamma + \frac{1}{\sqrt{n}}) \log^2 n$.
\subsection{Overview of our approach}
\label{sec:overview}
Let us recall the main parameters of our problem: the dataset size $n$, the range $[-R,R]$ of the function $L$, and the uniform stability parameter $\gamma$. Our approach is based on two operations, which reduce the bound on $D(n,R,\gamma)$ (ignoring for the moment the dependence on the tail probability $\delta$ as defined at the end of Section~\ref{sec:prelims}) to a bound on $D(n',R',\gamma)$ for some $n'$ and $R'$. For simplicity, we ignore some details and logarithmic factors in the following.

\paragraph{Range reduction.}
 If $\gamma < R / \sqrt{n}$, then by McDiarmid's inequality, $L(\bs,z)$ for each $z$ is concentrated in a window of size $R' = \gamma \sqrt{n} < R$. So we can ``center" the function by subtracting the function $\phi(z) = \E_\bs[L(\bs,z)]$  and then ``clamp" the function $L(\bs,z)-\phi(z)$ to range $R' = \gamma \sqrt{n}$. We will need to deal with two additional errors: the sampling error for $\phi$ which is on the order of $R/\sqrt{n}$, and the contribution of the values that we have ``clamped off".

 \paragraph{Dataset size reduction.}
For any setting of parameters, we can consider the dataset $[n]$ as partitioned into $b$ blocks of size $n' = n / b$. (For the moment ignoring divisibility issues.) Since the estimation error, $\frac1n \sum_{i=1}^{n} L(\bs,s_i)$, is an average over all coordinates, we can view it as an average of averages over each block. The expression for each block, conditioned on the variables outside the block, is just like the estimation error for a smaller problem with dataset size $n'$.
Again we obtain some additional error terms, but roughly speaking we can reduce the dataset size from $n$ to $n'$ without significant change in the estimation error.

\

Let us assume for now we can do both of these operations in a way that preserves the stability parameter $\gamma$ and keeps the function unbiased (by which we mean the condition $\E_z[L(\bs,z)] = 0$). Interestingly, applying these two operations repeatedly essentially proves the concentration inequality that we want. We sketch the argument below focusing on $\gamma = R / \sqrt{n}$. This is effectively the hardest regime (the result for other values of $\gamma$ is implied by applying one of the operations above once).

\begin{itemize}
    \item Suppose that the estimation error as a function of $n,R,\gamma$ is $D(n,R,\gamma)$. We want to prove that for $\gamma = R / \sqrt{n}$, $D(n,R,\gamma) = \tilde{O}(\gamma)$.

    \item Starting with parameters $n,R,\gamma$ such that $\gamma = R / \sqrt{n}$, we can use the block partitioning argument to decrease the dataset size from $n$ to $n' = n / b$ for some parameter $b>1$. The stability parameter is unchanged, and equal to $\gamma = R / \sqrt{b n'}$.

    \item Since $\gamma = R / \sqrt{b n'} < R / \sqrt{n'}$, we can use the range reduction argument to clamp the function $L(\bs,z)$ to a range of $R' = \gamma \sqrt{n'} = R / \sqrt{b}$. The stability parameter remains (hopefully) unchanged, $\gamma = R' / \sqrt{n'}$. Hence we are back in the situation similar to the one we started from, with the stability parameter being equal to the range divided by square root of the dataset size.

    \item Let's assume by induction that the estimation error for the function we obtained is $D(n',R',\gamma) = \tilde{O}(\gamma)$. By reversing the two operations that we performed, we obtain that, up to the additional error terms, the estimation error $D(n,R,\gamma)$ remains roughly the same, $D(n,R,\gamma) = \tilde{O}(\gamma))$. This leads to an inductive argument proving that the estimation error is indeed $\tilde{O}(\gamma)$.

\end{itemize}

\paragraph{Remaining issues.}
The main issue that need to be resolved in order to carry out this strategy is that simple ``clamping" to a fixed range $[-R',R']$ might cause the function to be no longer unbiased. The bias can be eliminated by subtracting the mean $\E_z[L(\bs,z)]$ again. Unfortunately, this operation may double the stability parameter. As a result, using this simple fix leads to a worse bound on the estimation error: $\gamma 2^{O(\sqrt{\log n})}$. To avoid this large overhead one needs to design a clamping operation that preserves both $\E_z[L(\bs,z)] = 0$ and the uniform stability parameter. It turns out that both requirements can be satisfied simultaneously by a $\bs$-dependent shift of the range, $[-R'+b_{\bs},R'+b_{\bs}]$ that however requires a rather delicate analysis. We present this construction in Section~\ref{sec:range-reduct}.

The unbiased property needs to be also preserved in the block partitioning operation. Here, we obtain it essentially ``for free", since the property $\forall \bs; \E_z[L(\bs,z)] = 0$ is preserved under conditioning on a subset of the coordinates in $\bs$. Thus the block partitioning operation is technically rather simple.

\iffull \else The details of dataset size reduction and the inductive argument are relatively simple and hence appear in Appendix \ref{app:proof}. \fi

%This is satisfied initially, but it can be violated by the clamping operation. We remark that when the clamping operation uses a fixed range $[-w,w]$, we need to use another centering operation to ensure that $\E_z[L(\bs,z)] = 0$,

%Another issue is the accumulation of errors from the various operations that we described. There are several candidate quantities to track throughout the proof, such as the exponential moment $\E[e^{\lambda \loE{L}}]$ or the maximum of $m$ independent samples $\E[\max\{\looE{L}{\bs_1}{z_1},\ldots,\looE{L}{\bs_m}{z_m} \}]$. In the end, the most convenient quantity proved to be the deviation $D_\delta(n,R,\gamma)$ defined by tail probability $\delta$.

%We now describe all the steps in detail.
%Our approach consists of two operations which we are going to apply repeatedly. The first operation reduces the range of $L(\bs,z)$, and the second operation reduces the dataset size $n$. Our goal is to prove that we can perform both operations without affecting the estimation error significantly.

\subsection{Range reduction}
\label{sec:range-reduct}

We start by designing a procedure which reduces the range of a function to a desired width while preserving the expectation and stability at the same time.

\begin{lem}
\label{lem:adaptclamp}
Let $K:Z^n \times Z \rightarrow [-r,r]$ be a function, $\cP$ a distribution on $Z$, and $w, \beta>0$ such that
\begin{compactitem}
\item for every $\bs \in Z^n$, $\E_{z \sim \cP}[K(\bs,z)] = 0$,
\item $\E_{(\bs,z) \sim \cP^{n+1}}[(K(\bs,z) - w)_+] \leq \beta$,
\item $\E_{(\bs,z) \sim \cP^{n+1}}[((-w-K(\bs,z))_+] \leq \beta$,
\item $K(\bs,z)$ has uniform stability $\gamma$.
\end{compactitem}
Then for every $\bs \in Z^n$, there exists $b_\bs \in [-w,w]$ such that $\tilde{K}(\bs,z) \doteq \clamp_{[b_\bs-w,b_\bs+w]}(K(\bs,z))$ has the following properties:
\begin{compactitem}
\item for every $\bs \in Z^n$, $\E_{z \sim \cP}[\tilde{K}(\bs,z)] = 0$,
\item $\E_{(\bs,z) \sim \cP^{n+1}}[|\tilde{K}(\bs,z) - K(\bs,z)|] \leq 4 \beta$,
\item $\tilde{K}(\bs,z)$ has uniform stability $\gamma$.
\end{compactitem}
\end{lem}

Note that in this lemma, the magnitude of the range $r$ does not play any role; it is just convenient to assume that $K(\bs,z)$ is bounded.

\begin{proof}
Fix $\bs \in Z^n$ and consider the function $\tilde{K}_x(\bs,z) \doteq \clamp_{[x-w,x+w]}(K(\bs,z))$. In other words,
$$ \tilde{K}_x(\bs,z) = {K}(\bs,z) - (K(\bs,z) - (x+w))_+ + ((x-w) - K(\bs,z))_+.$$
Therefore,
$$ \E_{z \sim \cP}[\tilde{K}_x(\bs,z)] = \E_{z \sim \cP}[{K}(\bs,z)] + \psi_\bs(x),$$
where
$$\psi_\bs(x) \doteq  - \E_{z \sim \cP}[(K(\bs,z) - (x+w))_+] + \E_{z \sim \cP}[((x-w) - K(\bs,z))_+].$$
$\psi_\bs(x)$ is continuous, since the expressions inside the expectations are uniformly continuous (in fact 1-Lipschitz) in $x$. Also, since $K(\bs,z) \in [-r,r]$ and $w > 0$, we have $\psi_\bs(-r) \leq 0$, $\psi_\bs(r) \geq 0$, and $\psi_\bs(x)$ is obviously non-decreasing. By the intermediate value theorem, there exists $b_\bs \in [-r,r]$ such that $\psi_\bs(b_\bs) = 0$. This means that $\E_{z \sim \cP}[(K(\bs,z) - (b_\bs+w))_+] =  \E_{z \sim \cP}[((b_\bs-w) - K(\bs,z))_+]$. We can define $\tilde{K}(\bs,z) \doteq \clamp_{[b_\bs-w,b_\bs+w]}(K(\bs,z))$ and we get $\E_{z \sim \cP}[\tilde{K}(\bs,z)] = \E_{z \sim \cP}[{K}(\bs,z)] = 0$. In addition, we observe that we must actually have $b_\bs \in [-w,w]$, because otherwise it would not be possible that a function bounded by $[b_\bs-w, b_\bs+w]$ has expectation $0$.

Let $\beta^+_\bs \doteq \E_{z \sim \cP}[(K(\bs,z) - w)_+]$ and $\beta^-_\bs \doteq \E_{z \sim \cP}[(-w-K(\bs,z))_+]$. If $b_\bs \geq 0$, then
$$\E_{z \sim \cP}[((b_\bs-w) - K(\bs,z))_+] = \E_{z \sim \cP}[(K(\bs,z) - (b_\bs+w))_+] \leq \beta^+_\bs.$$
Conversely, if $b_\bs \leq 0$, then
$$\E_{z \sim \cP}[(K(\bs,z) - (b_\bs+w))_+]  =\E_{z \sim \cP}[((b_\bs-w) - K(\bs,z))_+] \leq \beta^-_\bs.$$
Either way,
$$\E_{z \sim \cP}[(K(\bs,z) - (b_\bs+w))_+] + \E_{z \sim \cP}[((b_\bs-w) - K(\bs,z))_+] \leq \max \{2\beta^+_\bs, 2\beta^-_\bs\}.$$
Note that
$$|\tilde{K}(\bs,z) - K(\bs,z)| = (K(\bs,z) - (b_\bs+w))_+ + ((b_s-w) - K(\bs,z))_+.$$
Therefore, taking the expectation over both $\bs$ and $z$, we obtain
\begin{eqnarray*}
 \E_{(\bs,z) \sim \cP^{n+1}}[|\tilde{K}(\bs,z) - K(\bs,z)|] & =  &
\E_{(\bs,z) \sim \cP^{n+1}}[(K(\bs,z) - (b_\bs+w))_+] + \E_{(\bs,z) \sim \cP^{n+1}}[((b_s-w) - K(\bs,z))_+] \\
& = & \E_{\bs \sim \cP^n}[\E_{z \sim \cP}[(K(\bs,z) - (b_\bs+w))_+] + \E_{z \sim \cP}[((b_\bs-w) - K(\bs,z))_+]] \\
& \leq & \E_{\bs \sim \cP^n}[\max \{2 \beta^-_\bs, 2 \beta^+_\bs \}] \\
& \leq & \E_{\bs \sim \cP^n}[2\beta^-_\bs + 2\beta^+_\bs].
\end{eqnarray*}
We recall that by assumption,
$$ \E_{\bs \sim \cP^n}[\beta^+_\bs] = \E_{(\bs,z) \sim \cP^{n+1}}[(K(\bs,z) - w)_+] \leq \beta $$
and
$$ \E_{\bs \sim \cP^n}[\beta^-_\bs] = \E_{(\bs,z) \sim \cP^{n+1}}[(-w - K(\bs,z))_+] \leq \beta.$$
Hence we conclude that
$$  \E_{(\bs,z) \sim \cP^{n+1}}[|\tilde{K}(\bs,z) - K(\bs,z)|] \leq 4 \beta.$$

Finally, we need to argue about the stability of $\tilde{K}(\bs,z)$. We assume that $K(\bs,z)$ has stability $\gamma$. Suppose that $\bs'$ is obtained from $\bs$ by changing one coordinate. We claim that $|b_{\bs'} - b_\bs| \leq \gamma$. If not, suppose w.l.o.g.~that $b_{\bs'} > b_\bs + \gamma$: Then we would have
$$ \E_{z \sim \cP}[(K(\bs',z) - (b_{\bs'}+w))_+] < \E_{z \sim \cP}[((K(\bs,z)+\gamma) - (b_\bs+\gamma+w))_+]
 = \E_{z \sim \cP}[(K(\bs,z) - (b_\bs+w))_+].$$
By a similar argument,
$$ \E_{z \sim \cP}[((b_{\bs'}-w) - K(\bs',z))_+] > \E_{z \sim \cP}[((b_\bs+\gamma+w) - (K(\bs,z)+\gamma))_+]
 = \E_{z \sim \cP}[((b_\bs+w) - K(\bs,z))_+].$$
However, by construction the left-hand sides should be equal, and also the right-hand sides should be equal, which is a contradiction.

Now we can prove that by changing one coordinate of $\bs$, the value of $\tilde{K}(\bs,z)$ cannot change by more than $\gamma$. %Suppose $\bs'$ is obtained from $\bs$ by changing one coordinate.
Since both $K(\bs,z)$ and $b_\bs$ can change by at most $\gamma$ when switching from $\bs$ to $\bs'$, we have \alequn{\tilde{K}(\bs',z) &= \clamp_{[b_{\bs'}-w,b_{\bs'}+w]}(K(\bs',z)) \\& \geq \clamp_{[b_\bs-\gamma-w,b_\bs-\gamma+w]}(K(\bs,z)-\gamma)
\\ & = \clamp_{[b_\bs-w,b_\bs+w]}(K(\bs,z)) - \gamma \\& = \tilde{K}(\bs,z) - \gamma.} Similarly, we prove that $\tilde{K}(\bs',z) \leq \tilde{K}(\bs,z) + \gamma$. Therefore, $\tilde{K}(\bs,z)$ has stability $\gamma$.
\end{proof}

Next, we use adaptive clamping to argue that, roughly speaking, when $\gamma$ is small enough, we can reduce the range of $L(\bs,z)$ without changing $\gamma$ and affecting the estimation error significantly. An additional error term will appear due to the need to center the function $L(\bs,z)$ for each fixed $z$ before this operation; this additional error cannot be avoided.

\begin{lem}
\label{lem:range-reduction}
Let $n \geq 4$, $\delta \leq \frac{1}{e}$, and $\gamma, R>0$ such that $\gamma < \frac{R}{2 \sqrt{n \ln (n/\delta)}}$. Then
$$ D_{4\delta}(n,R,\gamma) \leq D_\delta(n,R',\gamma) + \frac{2R}{\sqrt{n}} \sqrt{\ln (1/\delta)} $$
where $R' = 2 \gamma \sqrt{n \ln (n/\delta)}$.
\end{lem}

\begin{proof}
Let $L:Z^n \times Z \rightarrow [-R,R]$ be a function of stability $\gamma$, $\cP$ a probability distribution over $Z$, and $\E_{z \sim \cP}[L(\bs,z)] = 0$.
First, let us shift the function for each fixed $z$ to make the expectation over $\bs$ equal to $0$.
We define $\phi(z) \doteq \E_{\bs \sim \cP^n}[L(\bs,z)]$ and $K(\bs,z) = L(\bs,z) - \phi(z)$.
Since for every $z \in Z$, $K(\bs,z)$ has sensitivity $\gamma$ in $\bs$ and $\E_{\bs \sim \cP^n}[K(\bs,z)] = 0$, McDiarmid's inequality (Lemma~\ref{lem:mcdiarmid}) implies that
$$ \pr_{\bs \sim \cP^n}[K(\bs,z) \geq \gamma \sqrt{n \ln (n/\delta)}] \leq \frac{\delta^2}{n^2},$$
$$ \pr_{\bs \sim \cP^n}[K(\bs,z) \leq -\gamma \sqrt{n \ln (n/\delta)}] \leq \frac{\delta^2}{n^2}.$$
Since the range of $K$ is bounded by $[-2R,2R]$, this implies that
$$ \E_{\bs \sim \cP^n}[(K(\bs,z) - \gamma \sqrt{n \ln (n/\delta)})_+] \leq \frac{2R \delta^2}{n^2},$$
$$ \E_{\bs \sim \cP^n}[(-\gamma \sqrt{n \ln (n/\delta)} - K(\bs,z))_+] \leq \frac{2R \delta^2}{n^2}.$$
Obviously the same bounds remain valid when we also take the expectation over $z \sim \cP$.

Next, we apply Lemma~\ref{lem:adaptclamp}, with $w = \gamma \sqrt{n \ln (n/\delta)}$, $r = 2R$ and $\beta = \frac{2R \delta^2}{n^2}$. Hence there exists $b_\bs \in [-w,w]$ for each $\bs \in Z^n$ such that $\tilde{K}(\bs,z) = \clamp_{[b_\bs-w,b_\bs+w]}(K(\bs,z))$ satisfies
\begin{compactitem}
\item for every $\bs \in Z^n$, $\E_{z \sim \cP}[\tilde{K}(\bs,z)] = 0$,
\item $\E_{(\bs,z) \sim \cP^{n+1}}[|\tilde{K}(\bs,z) - K(\bs,z)|] \leq 4 \beta = \frac{8R \delta^2}{n^2}$,
\item $\tilde{K}(\bs,z)$ has uniform stability $\gamma$.
\end{compactitem}
Also, since $b_\bs \in [-w,w]$, the function $\tilde{K}(\bs,z)$ is bounded by $[-2w,2w] = [-R',R']$, where $R' = 2w = 2 \gamma \sqrt{n \ln (n/\delta)}$ as in the statement of the lemma.

To summarize the relationship between $\tilde{K}(\bs,z)$ and $L(\bs,z)$,
we have
$$ L(\bs,z) = K(\bs,z) + \phi(z) = \tilde{K}(\bs,z) + \phi(z) + (K(\bs,z) - \tilde{K}(\bs,z)).$$
We want to bound the leave-one-out estimation error of $L(\bs,z)$,
$$ \loE{L} = \frac{1}{n} \sum_{i=1}^{n} {L}(\bs^{i \leftarrow z}, s_i). $$
From here, we can write
\begin{equation*}
\label{eq:Lterms}
\loE{L} = \loE{\tilde{K}} + \frac{1}{n} \sum_{i=1}^{n} \phi(s_i) + \frac{1}{n} \sum_{i=1}^{n} ({K}(\bs^{i \leftarrow z}, s_i) - \tilde{K}(\bs^{i \leftarrow z}, s_i))
\end{equation*}
and
\begin{equation*}
\label{eq:absLterms}
| \loE{L}| \leq |\loE{\tilde{K}}| + \left| \frac{1}{n} \sum_{i=1}^{n} \phi(s_i) \right| + \frac{1}{n} \sum_{i=1}^{n} \left|{K}(\bs^{i \leftarrow z}, s_i) - \tilde{K}(\bs^{i \leftarrow z}, s_i) \right|.
\end{equation*}

By the definition of $D_\delta$, since $\tilde{K}(\bs,z)$ has range $[-R',R']$ and uniform stability $\gamma$, we get
$$ \pr[|\loE{\tilde{K}}| \geq D_\delta(n,R',\gamma)] \leq \delta.$$

Let us bound the remaining two terms. The expression $\frac{1}{n} \sum_{i=1}^{n} \phi(s_i)$ is the average of $n$ independent samples in the range $[-R,R]$, which can be viewed as a function of $n$ independent random variables of sensitivity $\frac{2R}{n}$. Also, we know that $\E_{z \sim \cP}[\phi(z)] = \E_{\bs \sim \cP^n}\E_{z \sim \cP} [L(\bs,z)] = 0$. Therefore by Lemma~\ref{lem:mcdiarmid} (applied to $\frac{1}{n} \sum_i \phi(s_i)$ and $-\frac{1}{n} \sum_i \phi(s_i)$),
$$ \pr\lb \left|\frac{1}{n} \sum_{i=1}^{n} \phi(s_i) \right| \geq \frac{R}{\sqrt{n}} \sqrt{2 \ln \frac{1}{\delta}} \rb \leq 2 \delta.$$

Finally, the expression $\E_{(\bs,z) \sim \cP^{n+1}}[| {K}(\bs^{i \gets z}, s_i) - \tilde{K}(\bs^{i \gets z}, s_i)|]$ for each fixed $i$ is bounded by $4\beta = \frac{8R \delta^2}{n^2}$ as we argued above. Hence
$$\E_{(\bs,z) \sim \cP^{n+1}} \lb \frac1n \sum_{i=1}^{n} | {K}(\bs^{i \gets z}, s_i) - \tilde{K}(\bs^{i \gets z}, s_i)| \rb \leq \frac{8R \delta^2}{n^2}.$$
In this case, we simply use Markov's inequality, saying that
$$\pr_{(\bs,z) \sim \cP^{n+1}} \lb \frac1n \sum_{i=1}^{n} | {K}(\bs^{i \gets z}, s_i) - \tilde{K}(\bs^{i \gets z}, s_i)| \geq \frac{8 R \delta}{n^2} \rb \leq \delta.$$
Therefore, by (\ref{eq:absLterms}) and the union bound we obtain
\begin{eqnarray*}
\pr_{(\bs,z) \sim \cP^{n+1}} \lb\left| \loE{L} \right| \geq D_\delta(n,R',\gamma) + \frac{R}{\sqrt{n}} \sqrt{2 \ln \frac{1}{\delta}} + \frac{8R\delta}{n^2}  \rb \leq 4 \delta.
\end{eqnarray*}
For $\delta \leq 1/e$ and $n \geq 4$, we have $\frac{8R\delta}{n^2} \leq \frac{R}{e \sqrt{n}} \sqrt{\ln (1/\delta)}$, and $(\sqrt{2} + \frac{1}{e}) \frac{R}{\sqrt{n}} \sqrt{\ln (1/\delta)} < \frac{2R}{ \sqrt{n}} \sqrt{ \ln (1/\delta)}$.
Since this holds for every unbiased function $L(\bs,z)$ of range $[-R,R]$ and stability $\gamma$, we obtain the statement of the lemma.
\end{proof}

\iffull \else
\appendix
\section{Details of dataset size reduction and inductive argument}
\label{app:proof}
\fi

\subsection{Dataset size reduction}

Here we show by a block partitioning argument that increasing the dataset size $n$ cannot increase the estimation error significantly.

\begin{lem}
\label{lem:blocks}
For positive integers $k, n' \geq 1$, $n = k n'$, and real $R,\gamma>0$, $\delta>0$, let $L\colon Z^n \times Z \to [-R,R]$ be a data-dependent, $\gamma$-uniformly stable function unbiased relative a distribution $\cP$. Then
$$ D_\delta(n,R,\gamma) \leq D_{\delta/k}(n/k,R,\gamma).$$
\end{lem}

\begin{proof}
Assume first $n = k n'$.
Let $L(\bs,z)$ be any function as described in the lemma.
For a set of indices $I\subseteq [n]$ and $\bs\in Z^n$ we denote $\bs_I = (s_i)_{i\in I}$. Similarly, we denote
  $$\looE{L}{\bs_I}{z} \doteq \fr{|I|}\sum_{i\in I} L(\bs^{i\gets z},s_i) .$$
We partition the set $[n]$ into $k$ blocks of size $n'$: $B_1 = \{1,\ldots,n'\}, B_2 = \{n'+1,\ldots,2n'\}$, etc.
Observe that
\begin{equation}
\label{eq:block-reduct}
\looE{L}{\bs}{z} = \frac{1}{n} \sum_{j=1}^{k} \sum_{i \in B_j} L(\bs^{i \gets z}, s_i) = \frac{1}{k} \sum_{j=1}^{k} \looE{L}{\bs_{B_j}}{z}.
\end{equation}
If we condition on the values of $s_i$ for $i \notin B_j$, we can view the quantity $\looE{L}{\bs_{B_j}}{z}$ as the estimation error for the function $L$ restricted to the $n'$ variables $s_i, i \in B_j$. Hence, for any fixed choice of the values $s_i, i \notin B_j$, we have by definition
$$ \pr_{\bs_{B_j},z \sim \cP^{n'+1}}\lb \left|\looE{L}{\bs_{B_j}}{z}\right| \geq D_{\delta/k}(n',R,\gamma) \mid s_i: i \notin B_j\rb \leq \frac{\delta}{k}.$$
Since the bound is independent of the values of $s_i, i \notin B_j$, it remains valid if we remove the conditioning:
$$ \pr_{\bs_{B_j},z \sim \cP^{n'+1}}\lb \left|\looE{L}{\bs_{B_j}}{z}\right| \geq D_{\delta/k}(n',R,\gamma)\rb \leq \frac{\delta}{k}.$$
By (\ref{eq:block-reduct}) and the union bound,
$$ \pr_{\bs,z \sim \cP^{n+1}}\lb \left|\loE{L}\right| \geq D_{\delta/k}(n',R,\gamma)\rb \leq \pr \lb \exists j\in[k], \left|\looE{L}{\bs_{B_j}}{z}\right| \geq D_{\delta/k}(n',R,\gamma) \rb \leq \delta.$$
This means that $D_\delta(n,R,\gamma) \leq D_{\delta/k}(n',R,\gamma)$.
\end{proof}

We will need another version of this inequality, for the case where $n$ is not divisible by $n'$.

\begin{lem}
\label{lem:blocks2}
For positive integers $n \geq n'$, and real $R,\gamma>0$, $\delta>0$, let $L\colon Z^n \times Z \to [-R,R]$ be a data-dependent, $\gamma$-uniformly stable function unbiased relative a distribution $\cP$. Then
$$ D_\delta(n,R,\gamma) \leq D_{\delta/n}(n',R,\gamma).$$
\end{lem}

\begin{proof}
We use the same argument as above, except that we use $n$ overlapping blocks of size $n'$: $B_1 = \{1,\ldots,n'\}$, $B_2 = \{2, \ldots, n'+1\}$, etc. (using indices modulo $n$). Since each element appears in exactly $n'$ blocks, we obtain
$$ \looE{L}{\bs}{z} = \frac{1}{n' n} \sum_{j=1}^{n} \sum_{i \in B_j} L(\bs^{i \gets z}, s_i) = \frac{1}{n} \sum_{j=1}^{k} \looE{L}{\bs_{B_j}}{z}$$
instead of (\ref{eq:block-reduct}). The rest of the proof is exactly the same. We lose a factor of $n$ in the $\delta$ parameter because of a union bound over $n$ blocks.
\end{proof}

\subsection{The inductive bound on estimation error}

Here we combine the two reduction steps to prove our main bound on estimation error. It is convenient to state the inductive statement as follows.

\begin{lem}
\label{lem:inductive-bound}
For any $\delta \leq \frac{1}{e}$, $a \in \N$, $n = 4^a$, $\gamma = \frac{1}{\sqrt{n}}$ and $R = 8 \sqrt{\ln (n/\delta)}$,
$$ D_{\delta}(n,R,\gamma) \leq \frac{8}{\sqrt{n}} \ln \left( \frac{n^2}{\delta} \right) \log_2 n.$$
\end{lem}

\begin{proof}
We proceed by induction on $a$. The base case, $a=1$, holds trivially, because here $n = 4, \gamma = \frac12$, and the desired bound is $D_\delta(n,R,\gamma) \leq 8 \ln \frac{16}{\delta}$ which holds because $D_\delta(n,R,\gamma) \leq R = 8 \sqrt{\ln \frac{4}{\delta}}$.

For the inductive step, consider $n=4^{a+1} \geq 16$ and $\gamma=\frac{1}{\sqrt{n}} = \frac{1}{2^{a+1}}$. We use the two main ingredients that we proved above.

From the range reduction step (Lemma~\ref{lem:range-reduction}), since $\gamma = \frac{1}{\sqrt{n}} = \frac{R}{8 \sqrt{n \ln (n/\delta)}}$, we obtain
$$ D_{\delta}(n,R,\gamma) \leq D_{\delta/4}(n,R/4,\gamma) + \frac{2R}{\sqrt{n}} \sqrt{\ln \frac{4}{\delta}}. $$
Next, we consider $n' = n/4 = 4^a$, $\gamma' = 2 \gamma = \frac{1}{\sqrt{n'}}$ and $R' = 8 \sqrt{\ln(n' / \delta)}$. Using the basic scaling identity (\ref{eq:scale}) and $R' = 8 \sqrt{\ln (n/(2\delta))} \geq 4 \sqrt{\ln ({n}/{\delta})} = R/2$, we obtain
$$ D_{\delta}(n,R,\gamma) \leq \frac12 D_{\delta/4}(n,R/2,2\gamma) + \frac{2R}{\sqrt{n}} \sqrt{\ln \frac{4}{\delta}}
 \leq \frac12 D_{\delta/4}(n,R',\gamma') + \frac{16}{\sqrt{n}}  \ln \frac{n}{\delta}.$$
From the dataset size reduction step (Lemma~\ref{lem:blocks}), since $n = 4n'$, we obtain that
$$ D_{\delta/4}(n,R',\gamma') \leq D_{\delta/16}(n',R',\gamma').$$
Now we have an expression which is bounded by the inductive hypothesis:
$$D_{\delta/16}(n',R',\gamma') \leq \frac{8}{\sqrt{n'}} \ln \left(\frac{n'^2}{\delta/16} \right) \log_2 n'
 =  \frac{16}{\sqrt{n}} \ln \left(\frac{n^2}{\delta}\right) (-2 + \log_2 n).$$
Therefore we conclude that
$$ D_{\delta}(n,R,\gamma) \leq \frac12 D_{\delta/16}(n',R',\gamma') + \frac{16}{\sqrt{n}} \ln \left( \frac{n}{\delta} \right) \leq \frac{8}{\sqrt{n}} \ln \left( \frac{n^2}{\delta} \right) \log_2 n.$$
\end{proof}

From here, we can derive our main result by reducing it to Lemma \ref{lem:inductive-bound}. We first deal with
the case when $\gamma \geq \frac{1}{4 \sqrt{n \log (n/\delta)}}$.
\begin{thm}
\label{thm:main-large}
For any $\delta \leq \frac1e$, $n \geq 4$ and $\gamma \geq \frac{1}{4 \sqrt{n \log (n/\delta)}}$,
$$ D_\delta(n,1,\gamma) \leq 16 \gamma \ln \left( \frac{n^3}{\delta} \right) \log_2 n.$$
\end{thm}

\begin{proof}
Let us scale the function by a factor of $R = 4 \sqrt{\ln (n/\delta)}$, so we obtain a function with range $R$ and uniform stability $\gamma' = R \gamma \geq \frac{1}{\sqrt{n}}$.
Let $a' = \lfloor \log_4 (1/\gamma'^2) \rfloor$.
I.e, $n' = 4^{a'}$ is the largest power of $4$ below $1/\gamma'^2 \leq n$. Let also $\gamma'' = \frac{1}{\sqrt{n'}} \geq \gamma'$. Since the range is $R = 4 \sqrt{\ln (n/\delta)} \leq 8 \sqrt{\ln (n'/\delta)}$, by Lemma~\ref{lem:inductive-bound},
$$ D_\delta(n',R,\gamma'') \leq \frac{8}{\sqrt{n'}} \ln \left( \frac{n'^2}{\delta} \right) \log_2 n'
 = 8 \gamma'' \ln \left( \frac{n'^2}{\delta} \right) \log_2 n'.$$
Since $\gamma'' \geq \gamma'$ and $n \geq 1/\gamma'^2 \geq n'$, we get by monotonicity in $n$ (Lemma~\ref{lem:blocks2}) and monotonicity in $\gamma$ (obvious),
$$ D_\delta(n,R,\gamma') \leq D_{\delta/n}(n',R,\gamma'') \leq 8 \gamma'' \ln \left(\frac{n^3}{\delta} \right) \log_2 n. $$
Since $n'$ is within a factor of $4$ from $1/\gamma'^2$, we have $\gamma'' = \frac{1}{\sqrt{n'}} \leq 2 \gamma'$. So,
$$ D_\delta(n,R,\gamma') \leq 16 \gamma' \ln \left( \frac{n^3}{\delta} \right) \log_2 n.$$
Finally, scaling back by a factor of $1/R$ (see (\ref{eq:scale})), we conclude that
$$ D_\delta(n,1,\gamma) \leq 16 \gamma \ln \left( \frac{n^3}{\delta} \right) \log_2 n.$$
\end{proof}

We also obtain a bound for smaller values of $\gamma$.

\begin{thm}
\label{thm:main-small}
For any $\delta \leq \frac1e$, $n \geq 4$ and $\gamma < \frac{1}{4 \sqrt{n \ln (n/\delta)}}$,
$$ D_\delta(n,1,\gamma) \leq 16 \gamma \ln \left( \frac{4n^3}{\delta} \right) \log_2 n + \frac{2}{\sqrt{n}} \sqrt{\ln (4/\delta)}.$$
\end{thm}

\begin{proof}
Since $R=1$ and $\gamma < \frac{1}{4 \sqrt{n \ln (n/\delta)}}$, for $R' = 2 \gamma \sqrt{n \ln(n/\delta)}$ we obtain from Lemma~\ref{lem:range-reduction} that
$$ D_{\delta}(n,R,\gamma) \leq D_{\delta/4}(n,R',\gamma) + \frac{2}{\sqrt{n}} \sqrt{\ln (4/\delta)}.$$
Now $\gamma = \frac{R'}{2 \sqrt{n \ln (n/\delta)}}$, so we get from Theorem~\ref{thm:main-large} and the scaling identity (\ref{eq:scale}),
$$ D_{\delta/4}(n,R',\gamma) = R' \, D_{\delta/4}(n,1,\gamma/R') \leq 16 \gamma \ln \left( \frac{4n^3}{\delta} \right) \log_2 n.$$
\end{proof}

Finally we show how this implies Theorem~\ref{thm:main-intro}. Let $M:Z^n \times Z \rightarrow [0,1]$ be a data-dependent function of uniform stability $\gamma$. In Theorem~\ref{thm:main-intro}, we have the quantity
$$ \Delta_{\bs}(M) = \left| \E_{\cP}[M(\bs)] - \cE_{\bs}[M(\bs)] \right| = \left| \cE_{\bs}[L(\bs)] \right|$$
which differs by at most $2\gamma$ from the quantity $\left| \loE{L} \right|$, where $L(\bs,z)$ has uniform stability $2\gamma$ (see Section~\ref{sec:prelims}).
By definition, we have
$$ \pr\lb |\loE{L}| \geq D_\delta(n,1,2\gamma) \rb \leq \delta $$
and hence
$$ \pr\lb \Delta_{\bs}(M) \geq D_\delta(n,1,2\gamma) + 2\gamma \rb \leq \delta. $$
By Theorems~\ref{thm:main-large} and ~\ref{thm:main-small}, we have
\begin{eqnarray*}
D_\delta(n,1,2\gamma) + 2\gamma & \leq &
 32 \gamma \ln \left( \frac{4n^3}{\delta} \right) \log_2 n + \frac{2}{\sqrt{n}} \sqrt{\ln (4/\delta)} + 2 \gamma \\
 & \leq & 32 \gamma \ln \left( \frac{5n^3}{\delta} \right) \log_2 n + \frac{2}{\sqrt{n}} \sqrt{\ln (4/\delta)}.
\end{eqnarray*}
This proves Theorem~\ref{thm:main-intro}.

\section{Applications}
\label{sec:apps}
We now apply our bounds on the estimation error to several known uniformly stable algorithms. Additional applications can be derived in a similar manner.

\subsection{Learning via Stochastic Convex Optimization}
We consider learning problems that can be formulated as stochastic convex optimization. Specifically, these are problems in which the goal is to minimize the expected loss:
$$F_\cP(w) \doteq \E_{z\sim \cP}[\ell(w,z)],$$ over $w \in \K \subset \R^d$ for some convex body $\K$ and a family of convex losses $\F = \{\ell(\cdot,z)\}_{z\in Z}$. The stochastic convex optimization problem for $\F$  is the problem of minimizing $F_\cP(w)$ over $\K$ for an arbitrary distributions $\cP$ over $Z$. The {\em excess loss} of a vector $\tilde w$ is $F_\cP(\tilde w) - \min_{w \in \K} F_\cP(w)$. We also denote the empirical loss by $F_\bs(w) \doteq \fr{n} \sum_{i\in [n]}\ell(w,s_i)$.

Many learning problems can be expressed in or relaxed to this general form. As a result many optimization algorithms are known and the optimal (excess) error rates are understood for a variety of families of convex functions. Most of these results are obtained via algorithm-specific techniques such as online-to-batch conversion \citep{Cesa-BianchiCG04} and stability-based arguments rather than uniform convergence. As it turns out, this is unavoidable. This was first pointed out in the seminal work of \citet{ShwartzSSS10} who showed that there is exists a gap between the bounds that can be obtained via uniform convergence (or ERM algorithms) and bounds achievable via alternative approaches.

For concreteness, let $\F$ be the family of all convex $1$-Lipschitz losses over the unit Euclidean ball in $d$ dimension (denoted by $\B_2^d(1)$). It is well-known that in this case the stochastic convex optimization problem can be solved with expected excess error $1/\sqrt{n}$ via projected SGD. At the same time it was shown in \citep{ShwartzSSS10} that there exists an algorithm that minimizes the empirical loss while having the worst case excess loss of $\Omega\lp \frac{\log d}{n}\rp$. This has been subsequently strengthened to $\Omega\lp \frac{d}{n}\rp$ by \citet{Feldman:16erm} who also showed a lower bound of $\Omega\lp \sqrt{\frac{d}{n}}\rp$ for obtaining uniform convergence in this setting. Further, with Lipschitzness assumption replaced by the assumption that functions have range in $[0,1]$ the gap becomes infinite even for $d=2$ \citep{Feldman:16erm}.

\subsubsection{Strongly convex ERM}
We now revisit the stability results known for this basic setting \citep{BousquettE02,ShwartzSSS10} (for simplicity and without loss of generality we will scale the domain and functions to 1).
\begin{thm}[\citep{ShwartzSSS10}]
\label{thm:ssss}
Let $\K \subseteq \B_2^d(1)$ be a convex body, $\F = \{\ell(\cdot, z) \cond z\in Z\}$ be a family of $1$-Lipschitz, $\lambda$-strongly convex loss functions over $\K$ with range in $[0,1]$. For a dataset $\bs \in Z^n$ let $w_\bs$ denote the empirical minimizer of loss on $\bs$: $w_\bs =\argmin_{w \in \K} F_\bs(w)$. Then the algorithm $M(\bs,z)$ that evaluates $\ell(w_\bs,z)$ has uniform stability $\frac{4}{\lambda n}$.
\end{thm}
As an immediate corollary of this result and Theorem \ref{thm:main-intro} we obtain:
\begin{cor}
\label{cor:strongly-convex}
In the setting of Thm.~\ref{thm:ssss}, there exists a constant $c$ such that for every $\delta >0$:
$$\pr_{\bs\sim \cP^n} \lb F_\cP(w_{\bs}) \geq \min_{w \in \K} F_\cP(w) + \frac{c \log(n) \log(n/\delta)}{\lambda n} + \frac{c \sqrt{\log (1/\delta)}}{\sqrt{n}}\rb \leq \delta .$$
\end{cor}

Theorem ~\ref{thm:ssss} requires strong convexity. As pointed out in \citep{ShwartzSSS10}, it is possible to add a strongly convex regularizing term $\frac{\lambda}{2} \|w\|^2$ to the objective function that has sufficiently small effect on the loss function while ensuring stability (and generalization). Specifically, the objective function will change by at most $\lambda$ since $w$ is assumed to be in a ball of radius $1$. By choosing $\lambda = \frac{\log n}{\sqrt{n}}$, we obtain a excess loss of $O\left( \frac{\log(n/\delta)}{\sqrt{n}}\right)$ (Corollary \ref{cor:convex-general}). This improves on the $O\lp\frac{\sqrt{\log (1/\delta)}}{n^{1/3}}\rp$ bound on excess loss obtained from the results in \citep{FeldmanV:18} by choosing $\lambda = 1/n^{2/3}$. We also remark that it is well-known that the same (up to a constant factor) stability bounds --- and hence generalization bounds -- apply to algorithms that minimize the (regularized) empirical loss within $1/n$. Therefore Corollary \ref{cor:strongly-convex} leads to an efficient algorithm for solving the problem.

\subsubsection{(Deterministic) gradient descent}
We now recall the results of \citet{HardtRS16} for convex and smooth functions. These results derive their guarantees from the fact that a gradient step on a sufficiently smooth loss function is non-expansive. That is, for any pair of points $w$ and $w'$, any $\sigma$-smooth (that is, having a $\sigma$-Lipschitz gradient) convex function $f$, and $0 \leq \eta \leq 2/\sigma$, \equ{\|(w - \eta \nabla f(w)) - (w' - \eta \nabla f(w'))\| \leq   \|w - w'\| .\label{eq:contraction}}
Projection to a convex body is also non-expansive. This implies that uniform stability can be proved for projected gradient descent of the following general form.
For a vector $\bbeta_t = (\eta_{t,1},\ldots,\eta_{t,n})$ a gradient step with rate vector $\bbeta_t$ is the update
\equ{w_{t+1} \leftarrow \proj_\K\lp w_t - \sum_{i\in [n]} \eta_{t+1,i} \nabla \ell(w_t,s_i) \rp ,\label{eq:pgd-step}}
where $\proj_\K$ denotes projection to $\K$. For example, if a batch of size $k$ is used in a gradient step with rate $\eta_t$ then for each point $s_i$ in the batch $\eta_{t,i} = \eta_t/k$ and for each point not in the batch $\eta_{t,i} = 0$.
The non-expansiveness of the gradient steps and projections implies that the effect of each datapoint $s_i$ on the loss of the solution can be bounded by $\sum_t \eta_{t,i} \|\nabla \ell(w_t,s_i) \|_2$. More formally, the following lemma follows directly from eq.~\eqref{eq:contraction} (and is a simple generalization of analysis in \citep{HardtRS16} that only considers updates on a single data sample). We include the proof for completeness.
\begin{lem}
\label{lem:stability-bounded-by-rate}
Let $\K \subseteq \B_2^d(1)$ be a convex body, $\F = \{\ell(\cdot, z) \cond z\in Z\}$ be a family of convex  $1$-Lipschitz and $\sigma$-smooth loss functions over $\K$ with range in $[0,1]$.  For a dataset $\bs$, number of iterations $T$ and a sequence of rate vectors $\bbeta_1,\ldots, \bbeta_T$
 let PGD$(w_0,(\bbeta_t)_{t\in [T]},\bs)$ denote the output of the algorithm that starting from $w_0 \in \K$, performs $T$ updates according to eq.~\eqref{eq:pgd-step} and returns $w_T$. If for every $t\in [T]$, $\eta_t \doteq \|\bbeta_t\|_1 \leq 2/\sigma$ then, for every $w_0$, the algorithm $M(\bs,z)$ that evaluates $\ell(w_T,z)$ on the output $w_T$ of PGD$(w_0,(\bbeta_t)_{t\in [T]},\bs)$ has uniform stability $2 \cdot \left\|(\bbeta)_{t\in [T]}\right\|_{1,\infty}$, where
  $$\left\|(\bbeta)_{t\in [T]}\right\|_{1,\infty} \doteq \max_{i \in [n]} \sum_{t\in [T]} \eta_{t,i} .$$
\end{lem}
\begin{proof}
Let $\bs$ and $\bs'$ be two datasets that differ in a single element at index $i^\ast$. Let $w_T=\mbox{PGD}(w_0,(\bbeta_t)_{t\in [T]},\bs)$ and $w'_T=\mbox{PGD}(w_0,(\bbeta_t)_{t\in [T]},\bs')$. We prove the following claim by induction on $T$.
$$\|w_T - w'_T\|_2 \leq 2 \sum_{t\in [T]} \eta_{t,i^\ast} .$$
The lemma will then follow from the definition of stability and 1-Lipschitness of $\ell(w_T,z)$.
Clearly, the claim holds for $T=0$.
Now, assume that the claim holds for all $T \leq \tau$.
\alequn{\|w_{\tau+1} - w'_{\tau+1}\|_2 &= \left\| \proj_\K\lp w_\tau  - \sum_{i\in [n]} \eta_{\tau+1,i} \nabla \ell(w_\tau,s_i) \rp - \proj_\K\lp w'_\tau - \sum_{i\in [n]} \eta_{\tau+1,i} \nabla \ell(w'_\tau,s'_i) \rp \right\|_2 \\
& \leq \left\| \lp w_\tau  - \sum_{i\in [n]} \eta_{\tau+1,i} \nabla \ell(w_\tau,s_i) \rp - \lp w'_\tau - \sum_{i\in [n]} \eta_{\tau+1,i} \nabla \ell(w'_\tau,s'_i) \rp \right\|_2  \\  & \leq \left\| \lp w_\tau - \sum_{i\in [n]\setminus\{i^\ast\}} \eta_{\tau+1,i} \nabla \ell(w_\tau,s_i) \rp -\lp w'_\tau - \sum_{i\in [n]\setminus\{i^\ast\}} \eta_{\tau+1,i} \nabla \ell(w'_\tau,s_i) \rp \right\|_2
\\ & \ \ \ + \left\| \eta_{\tau+1,i^\ast} \nabla \ell(w_\tau,s_{i^\ast}) - \eta_{\tau+1,i^\ast} \nabla \ell(w_\tau,s'_{i^\ast}) \right\|_2 \\
& \leq \|w_\tau  -w'_\tau\|_2 + \left\| \eta_{\tau+1,i^\ast} \nabla \ell(w_\tau,s_{i^\ast}) - \eta_{\tau+1,i^\ast} \nabla \ell(w_\tau,s'_{i^\ast}) \right\|_2  \\
& \leq 2 \sum_{t \in [\tau]} \eta_{t,i^\ast}  + 2\eta_{\tau+1,i^\ast} = 2 \sum_{t \in [\tau+1]} \eta_{t,i^\ast},
}
where we used eq.~\eqref{eq:contraction} for gradient step at rate $\eta_{\tau+1}=\|\bbeta_{\tau+1}\|_1$ on the function $$f(w)=\sum_{i\in [n]\setminus\{i^\ast\}}\frac{\eta_{\tau+1,i}}{\eta_{\tau+1}} \ell(w,s_i) $$ to obtain the fifth line. Note that $f$ is a convex combination of functions from $\F$ and therefore is $\sigma$-smooth and by our assumption $\eta_{\tau+1} \leq 2/\sigma$.
\end{proof}

Lemma \ref{lem:stability-bounded-by-rate} together with Theorem \ref{thm:main-intro} immediately implies generalization bounds for a variety of versions of gradient descent with different rates, arbitrary batch sizes and multiple passes over the data. For most such algorithms no alternative analyses of estimation error are known. Importantly, the estimation error can be bounded without any assumptions on how close the output of the algorithm is to the empirical minimum. Therefore this approach gives generalization bounds for algorithms used in practice as opposed to rates and bounds on number of iterations that are necessary for a theoretical proof of convergence (but are rarely used in practice).

As a concrete example we give a corollary for running full gradient descent with standard rates that guarantee convergence to within $1/\sqrt{n}$ of the empirical minimum. We are not aware of any other approaches to proving generalization guarantees for this algorithm in this general setting. Let PGD$(T,\eta,\bs)$ denote PGD$(w_0,(\bbeta_t)_{t\in [T]},\bs)$ for $w_0$ being the origin and $\eta_{t,i} = \eta/n$ for all $i\in [n]$ and $t\in [T]$. Standard analysis of gradient descent on $\sigma$-smooth functions (\eg \cite{Bubeck15}) implies that in the setting of Lemma \ref{lem:stability-bounded-by-rate}, PGD$(T,1/\sigma,\bs)$ outputs $w_{\bs}$ such that
\equ{F_\bs(w_{\bs}) \leq \min_{w \in \K} F_\bs(w) + \frac{2}{\eta T} .\label{eq:smooth-gd}}
By optimizing the choice of $T$, the best previous bound in \citep{FeldmanV:18} gives an upper bound of $O\lp\frac{\sqrt{\log (1/\delta)}}{n^{1/3}}\rp$ on excess loss. Similarly, applying our improved bounds gives the following statement.
\begin{cor}
\label{cor:smooth}
Let $\K \subseteq \B_2^d(1)$ be a convex body, $\F = \{\ell(\cdot, z) \cond z\in Z\}$ be a family of convex  $1$-Lipschitz and $\sigma$-smooth loss functions over $\K$ with range in $[0,1]$. For every distribution $\cP$ over $Z$, $\delta > 0$, $w_{\bs} \doteq \mbox{PGD}(T,\eta,\bs)$ for $\eta = 1/\sigma$ and $T=\lfloor \sigma \sqrt{n}/\log n \rfloor$, and some fixed constant $c$
$$\pr_{\bs\sim \cP^n} \lb F_\cP(w_{\bs}) \geq \min_{w \in \K} F_\cP(w) + \frac{c\log (n/\delta)}{\sqrt{n}}\rb \leq \delta .$$
\end{cor}
\begin{proof}
We first note that, by Lemma \ref{lem:stability-bounded-by-rate}, $\mbox{PGD}(T,\eta,\bs)$ with $\eta = 1/\sigma$ and $T=\lfloor \sigma \sqrt{n}/\log n \rfloor$ is $\frac{2}{\sqrt{n} \log n}$ uniformly stable (here we can assume that $\sigma \geq \log n/\sqrt{n}$ since otherwise $T=0$ and $w_0$ has the desired property since the range of every loss function is within $\log n/\sqrt{n}$ of some constant).

We next denote $w^\ast \doteq \argmin_{w \in \K} F_\cP(w)$ and use the following standard decomposition of excess loss:
$$F_\cP(w_\bs) - F_\cP(w^\ast) \leq |F_\cP(w_\bs) - F_\bs(w_\bs)| + \left|F_\bs(w_\bs) -  \min_{w \in \K} F_\bs(w)\right| +  \min_{w \in \K} F_\bs(w)  - F_\cP(w^\ast).$$
Theorem \ref{thm:main-intro} gives an upper-bound of $\frac{c_0\log (n/\delta)}{\sqrt{n}}$ (for some constant $c_0$) on the first term that holds with probability $1-\delta/2$. Equation \eqref{eq:smooth-gd} upper bounds the second term by $\frac{4 \log n}{\sqrt{n}}$ (we use an additional factor of 2 to account for the $\lfloor \cdot \rfloor$ operation). Finally, $\min_{w \in \K} F_\bs(w)$ has sensitivity of $1/n$ and $$\E_{\bs \sim \cP^n} \lb  \min_{w \in \K} F_\bs(w) \rb \leq \E_{\bs \sim \cP^n} \lb  F_\bs(w^\ast) \rb = F_\cP(w^\ast) .$$ Therefore, by McDiarmid's inequality (Lemma \ref{lem:mcdiarmid}),
$$\pr_{\bs \sim \cP^n}\lb \min_{w \in \K} F_\bs(w)  - F_\cP(w^\ast) \geq \frac{\sqrt{2\ln(2/\delta)}}{\sqrt{n}} \rb \leq \delta/2 .$$
Combining the upper bounds on the three terms and using the union bound we obtain the claim.
\end{proof}

\subsubsection{Stochastic gradient descent}
The analysis above applies to gradient descent with the rates chosen deterministically. In practice, a variety of randomized strategies for picking the batches are used with the most common ones being random shuffling and random sampling with replacement. We describe a strategy for picking which samples to use by a distribution $\cU$ over sequences of rate vectors that result from this strategy. We also denote by PSGD$(w_0,\cU,\bs)$ the corresponding stochastic gradient descent algorithm: sample $(\bbeta_t)_{t\in [T]}$ from $\cU$ and run PGD$(w_0,(\bbeta_t)_{t\in [T]},\bs)$.

A simple way to obtain generalization bounds for stochastic gradient descent is to bound the stability of the expectation (over the choice of batches) of the loss function. One can directly upper-bound it by the expectation of the uniform stability parameter $2 \E_\cU\lb \left\|  (\bbeta_t)_{t\in [T]} \right\|_{1,\infty}\rb$. For (multi-pass) sampling without replacement that is symmetric with respect to the samples, this immediately gives
$$\E_\cU\lb \left\|  (\bbeta_t)_{t\in [T]} \right\|_{1,\infty}\rb = \fr{n} \sum_{t\in [T]} \eta_t,$$ where $\eta_t$ is the rate of the batch used at step $t$. For sampling of batches with replacement, a direct application of this bound will not give the same result due to the (likely) repetition of samples. However, in this case the randomness in each iteration is independent and therefore one can take the expectation in every step of the induction in Lemma \ref{lem:stability-bounded-by-rate}. This leads to the same bound on the uniform stability parameter of the expected loss:
$$2 \max_{i \in [n]} \sum_{t\in [T]} \E_\cU[ \eta_{t,i}] = 2 \fr{n} \sum_{t\in [T]} \eta_t .$$ Combining this simple analysis with Theorem \ref{thm:main-intro}, one gets generalization with high probability over the dataset but in expectation over the sampling of batches.

To obtain generalization bounds that also hold with high probability over the sampling of the batches, we observe that for most common sampling schemes, $\left\|(\bbeta_t)_{t\in [T]}\right\|_{1,\infty}$ is highly concentrated around its expectation. In particular, the norm can be upper-bounded with high probability with relatively low overhead. More formally, we state the following general form of bounds on the estimation error of PSGD.
\begin{thm}
\label{thm:psgd-general}
Let $\K \subseteq \B_2^d(1)$ be a convex body, $\F = \{\ell(\cdot, z) \cond z\in Z\}$ be a family of convex  $1$-Lipschitz and $\sigma$-smooth loss functions over $\K$ with range in $[0,1]$.  For a number of iterations $T$ let $\cU$ be a distribution over sequences of rate vectors of length $T$ and assume that for every $(\bbeta_t)_{t\in [T]}$ in the support $\cU$, $\|\bbeta_t \|_1 \leq 2/\sigma$ for all $t\in [T]$.

Assume that for some $\beta \geq 0$, $$\pr_{(\bbeta_t)_{t\in [T]} \sim \cU}  \lb \left\|(\bbeta_t)_{t\in [T]}\right\|_{1,\infty} \geq \zeta \rb \leq \beta .$$
Then there exist a constant $c$ such that for every distribution $\cP$ over $Z$ and $w_0 \in \K$,
$$\pr_{\bs \sim \cP^n,\ w_{\bs,T} = \mbox{PSGD}(w_0,\cU,\bs)} \lb  \left| F_\cP(w_{\bs,T}) -  F_\bs(w_{\bs,T}) \right| \geq c\lp \zeta \log(n) \log (n/\delta)   + \frac{\sqrt{\log (1/\delta)}}{\sqrt{n}} \rp \rb \leq \beta + \delta .$$
\end{thm}

\paragraph{Random shuffling:} We first consider random shuffling based schemes (also referred to as sampling without replacement). In such schemes the dataset is split into batches randomly and uniformly. All the batches are used to update the gradient in every pass over the dataset. For every pass over the data each of the samples is used exactly once. Hence the contribution of each pass over the data to the stability parameter is upper bounded by the largest rate used in that pass. Specifically, if the batch size is $k$, $r$ passed are performed, and in each pass $i$ the largest rate used for a batch is $\eta_i$, then for $\zeta = \fr{k} \sum_{i\in [r]} \eta_i$, $\left\|(\bbeta_t)_{t\in [T]}\right\|_{1,\infty}$ is upper bounded by $\zeta$ with probability 1. In particular, Theorem \ref{thm:psgd-general} can be applied with $\zeta$ defined as above and $\beta =0$. While our results give a bound on the estimation error, unfortunately very little is known about the empirical error of gradient descent with random shuffling. In particular, known results for random shuffling are in more restrictive settings \citep{RechtR12,gurbuzbalaban2015random,Shamir16,LinR16,PillaudRB18} and in most cases only apply to function classes simple enough that one can appeal to complexity-based generalization bounds instead of stability.

\paragraph{Sampling with replacement:} Another sampling scheme (more common in theoretical analyses than in practice) uses random and independent sampling with replacement: that is for every iteration a batch of $k$ samples is chosen randomly, uniformly and independently of previous batches. For each of the $T$ iterations, sample $s_i$ is included with probability $k/n$ and therefore the sum of rates for sample $i$ is distributed as $\fr{k}\sum_{t\in [T]} \eta_t B(k/n)$, where $B(k/n)$ is the Bernoulli random variable with bias $k/n$. We can now use standard concentration inequalities and the union bound to upper bound the largest sum of rates. For example, if $T = O(n/k)$ (which corresponds to a constant number of passes) and the rate is fixed to $\eta$ then, by the (multiplicative) Chernoff bound, for some constant $c_0$,
\equ{\pr_{(\bbeta_t)_{t\in [T]} \sim \cU}  \lb \left\|(\bbeta_t)_{t\in [T]}\right\|_{1,\infty} \geq \frac{c_0 \eta \log(n/\beta)}{k} \rb \leq \beta .\label{eq:high-prob-rate}}
Hence even with a constant number of passes and batches of size 1 the overhead of getting generalization with high probability over the randomness of the algorithm is at most logarithmic. The overhead becomes (relatively) smaller as the number of passes grows.

As a concrete corollary, we give high-probability generalization bounds for PSGD  with $k=1$ and fixed rate $\eta = 1/\sqrt{T}$. We denote this rate distribution by $\cU_{1,T}$ and denote the origin in $\R^d$ by $\bar 0$ (and thus the algorithm can is exactly $\mbox{PSGD}(\bar 0,\cU_{1,T},\bs)$).
To get a high-probability bound on the empirical loss of this algorithm, we note that sampling with replacement corresponds to drawing i.i.d.~samples from the uniform distribution over the samples in $\bs$. In particular, the expected loss function in this case is exactly $F_\bs$. Standard high-probability generalization bounds for PSGD imply that it  minimizes the empirical loss with high-probability but require outputting the average of the iterates (these results are obtained via online-to-batch conversion \citep{Cesa-BianchiCG04}). Theorem \ref{thm:psgd-general} applies to the average of the iterates since Lemma \ref{lem:stability-bounded-by-rate} applies to it as well (or any other convex combination of the iterates). To get an upper-bound for $\mbox{PSGD}(\bar 0,\cU_{1,T},\bs)$ (which outputs the last iterate) we use a recent work of \citet{HarvetLPR18} that shows high-probability bound on suboptimality of the last iterate of PSGD\footnote{The results stated there are for the decaying rate $\eta_t = 1/\sqrt{t}$, but the same result applies to the fixed rate we use here \citep{Harvey18:PC}.} with a slightly worse rate.
\begin{lem}\citep{HarvetLPR18,Harvey18:PC}
\label{lem:psgd-emp}
Let $\K \subseteq \B_2^d(1)$ be a convex body, $\F = \{\ell(\cdot, z) \cond z\in Z\}$ be a family of convex  $1$-Lipschitz  loss functions over $\K$ with range in $[0,1]$. There exists a constant $c$ such that for every $\bs \in Z^n$ and $\delta >0$,
$$\pr_{w_T = \mbox{PSGD}(\bar 0,\cU_{1,T},\bs)} \lb F_\bs(w_{T}) \geq \min_{w \in \K} F_\bs(w) + \frac{c \log (T) \log (1/\delta)}{\sqrt{T}}\rb \leq \delta .$$
\end{lem}
Combining Lemma \ref{lem:psgd-emp} with Theorem \eqref{thm:psgd-general} (used with eq.~\ref{eq:high-prob-rate}), we obtain the following bound on the generalization error of PSGD$(\bar 0,\cU_{1,T},\bs)$ for $T=n$.
\begin{cor}
\label{cor:psgd-resample}
Let $\K \subseteq \B_2^d(1)$ be a convex body, $\F = \{\ell(\cdot, z) \cond z\in Z\}$ be a family of convex  $1$-Lipschitz $2\sqrt{n}$-smooth loss functions over $\K$ with range in $[0,1]$. There exists a constant $c$ such that for every distribution $\cP$ over $Z$ and $\delta >0$,
$$\pr_{\bs\sim \cP^n,\ w_n = \mbox{PSGD}(\bar 0,\cU_{1,n},\bs)} \lb F_\cP(w_{T}) \geq \min_{w \in \K} F_\cP(w) + \frac{c \log(n) \log^2 (n/\delta)}{\sqrt{n}}\rb \leq \delta .$$
\end{cor}

\begin{rem}
\label{rem:smoothing}
Finally, we note that the results in this section can be extended to non-smooth functions by applying a smoothing operation to each convex loss function before optimization. A variety of approaches to smoothing are known (\eg \citep{beck2012smoothing}). For our purposes it suffices to observe that for every convex $1$-Lipschitz function $f$ over $\K$ of radius 1, the standard smoothing via Moreau envelope can be used to obtain a $\sigma$-smooth 1-Lipschitz function $\tilde f$ such that $|\tilde{f}(w) - f(w)| \leq 1/(2\sigma)$ for all $w \in \K$. Thus we can apply the optimization to $\sqrt{n}$-smooth loss functions that are within $1/\sqrt{n}$ (in $L_\infty$ norm) of the corresponding functions in $\F$. Note that this level of smoothness and additional error suffice to extend Corollary \ref{cor:smooth} (with $T = n/\log n$) and Corollary \ref{cor:psgd-resample} to non-smooth functions with essentially the same bound on the excess loss.
\end{rem}

%By using $\lambda = 1/\sqrt{dn}$ we will obtain $\tilde F$ that is within $1/\sqrt{n}$ of $F_\bs$ and is $\sqrt{dn}$-smooth.

% (\eg \citep{FeldmanMTT18})
\remove{
Note that Corollary \ref{cor:smooth} applies to $\F$ with an arbitrary finite bound on the smoothness parameter $\sigma$. However the number of iterations grows linearly with $\sigma$. To ensure an upper bound on the running time and also to handle the non-smooth functions one can always apply a smoothing operation to $F_\bs$. For example, the standard smoothing via convolution with a Gaussian kernel of variance $d \lambda^2$  gives a $1/\lambda$-smooth function $\tilde F$ such that $|\tilde{F}(w) - F_\bs(w)| \leq \lambda \sqrt{d}$ for all $w \in \K$ (\eg \citep{FeldmanMTT18}). By using $\lambda = 1/\sqrt{dn}$ we will obtain $\tilde F$ that is within $1/\sqrt{n}$ of $F_\bs$ and is $\sqrt{dn}$-smooth. Hence for non-smooth loss functions we obtain the following corollary.
\begin{cor}
\label{cor:non-smooth}
Let $\K \subseteq \B_2^d(1)$ be a convex body, $\F = \{\ell(\cdot, z) \cond z\in Z\}$ be a family of convex  $1$-Lipschitz loss functions over $\K$ with range in $[0,1]$. Let $\mbox{PGD}_\lambda(T,\eta,\bs)$ denote the execution of $\mbox{PGD}(T,\eta,\bs)$ on $F_\bs$ smoothed via a convolution with Gaussian kernel $N(0,\lambda^2 \mathbb{I}_d)$. For every distribution $\cP$ over $Z$, $\delta > 0$, and $w_{\bs} \doteq \mbox{PGD}_\lambda(T,\eta,\bs)$ for $\lambda = \eta =1/\sqrt{dn}$, $T=\lfloor n \sqrt{d}/\log n \rfloor$, and some fixed constant $c$
$$\pr_{\bs\sim \cP^n} \lb F_\cP(w_{\bs}) \geq \min_{w \in \K} F_\cP(w) + \frac{c\log (n/\delta)}{\sqrt{n}}\rb \leq \delta .$$
\end{cor}
}

\subsection{Privacy-Preserving Prediction}
\label{sec:dp-api}
Our results can also be used to improve the bounds on generalization error of learning algorithms with differentially private prediction. These are algorithms introduced to model  privacy-preserving learning in the settings where users only have black-box access to the model via a prediction interface \citep{DworkFeldman18}.
Formally,
\begin{defn}[\citep{DworkFeldman18}]
\label{def:private-prediction}
Let $K$ be an algorithm that given a dataset $\bs\in (X\times Y)^n$ and a point $x\in X$ produces a value in $Y$. Then  $K$ is {\em  $\eps$-differentially private prediction} algorithm if for every $x \in X$, the output $K(\bs,x)$ is $\eps$-differentially private with respect to $\bs$.
\end{defn}
The properties of differential privacy imply that the expectation over the randomness of $K$ of the loss of $K$ at any point is uniformly stable. Specifically, for every $\eps$-differentially private prediction algorithm, every loss function $\ell_Y\colon Y\times Y \to [0,1]$, two datasets $\bs$ and $\bs'$ that differ in a single element and $(x,y) \in X\times Y$ we have that
$$\E_K[\ell_Y(K(\bs,x),y)] \leq e^\eps \cdot \E_K[\ell_Y(K(\bs',x),y)] .$$
In particular, this implies that
$$\left| \E_K[\ell_Y(K(\bs,x),y)] - \E_K[\ell_Y(K(\bs',x),y)] \right| \leq e^\eps-1 .$$
Therefore our generalization bounds can be applied to the data-dependent function $M(\bs,(x,y)) \doteq \E_K[\ell_Y(K(\bs,x),y)]$. This gives the following corollary of Theorem \ref{thm:main-intro}:
\begin{thm}
\label{thm:dp-api}
For $\eps \in (0,1)$, let $K:(X\times Y)^n \times X \to Y$ be an $\eps$-differentially private prediction and $\ell_Y\colon Y\times Y \to [0,1]$ be an arbitrary loss function. Let $M(\bs,(x,y)) \doteq \E_K[\ell_Y(K(\bs,x),y)]$. Then there exists a constant $c$ such for any probability distribution $\cP$ over $Z$ and any $\delta \in (0,1)$:
  \equn{\pr_{\bs \sim \cP^n} \lb \left| \E_{\cP}[M(\bs)] - \cE_{\bs}[M(\bs)] \right| \geq c \eps \log (n) \log(n/\delta) + \frac{\sqrt{2\ln (4/\delta)}}{\sqrt{n}} \rb \leq \delta .}
\end{thm}
These bounds are stronger than those obtained in \citep{DworkFeldman18} in several parameter regimes (but are more generally incomparable since bounds in \citep{DworkFeldman18} are multiplicative).

\citet{DworkFeldman18} describe an algorithm for agnostically learning threshold functions on a line with differentially private prediction. Their analysis of the generalization error of this algorithm relies crucially on the generalization properties of differentially private prediction. Their weaker generalization bound does not give the high-probability bound on the generalization error that is necessary for satisfying the standard definition of agnostic learning. By plugging in Thm.~\ref{thm:dp-api} we obtain a bound on generalization error that holds with high probability and achieves the optimal rate (up to logarithmic factors). We omit more formal details since they require several additional definitions and the application itself is straightforward.

\subsubsection*{Acknowledgments}
We thank Nick Harvey, Tomer Koren, Mehryar Mohri, Sasha Rakhlin, Yoram Singer, Karthik Sridharan, Csaba Szepesvari and Kunal Talwar for thoughtful discussions and insightful comments about this work.

\iffull
\printbibliography
\else
\bibliography{vf-allrefs-local,stable}
\fi

\end{document}